\documentclass[11pt,a4paper]{article}
\usepackage[hyperref]{emnlp2020}
\usepackage{times}
\usepackage{latexsym}

\usepackage{booktabs}

\usepackage{microtype}

\aclfinalcopy 
\def\aclpaperid{1888} 

\usepackage{hyperref}       
\usepackage{url}            
\usepackage{booktabs}       
\usepackage{nicefrac}       
\usepackage{microtype}      
\usepackage{multirow}
\usepackage{subfig}
\usepackage{enumerate}
\usepackage{caption}
\usepackage{graphicx}
\usepackage{xspace}
\usepackage{paralist}
\usepackage{float}
\usepackage{scrextend}
\usepackage[tight-spacing=true]{siunitx}
\usepackage{color,soul}
\usepackage{etoolbox,multirow,adjustbox}
\usepackage{comment}
\usepackage{bbm}

\graphicspath{{figures/}}
\usepackage{mathtools,commath,amsmath,amssymb,amsfonts,bm}\interdisplaylinepenalty=2500
\usepackage{amsthm}
\usepackage{algpseudocode,algorithm}
\usepackage{tabularx,adjustbox}
\newcolumntype{L}[1]{>{\hsize=#1\hsize\raggedright\arraybackslash}X}%
\newcolumntype{R}[1]{>{\hsize=#1\hsize\raggedleft\arraybackslash}X}%
\newcolumntype{C}[1]{>{\hsize=#1\hsize\centering\arraybackslash}X}%
\usepackage{todonotes}
\usepackage{breqn}
\clearpage
\extrafloats{100}
\RequirePackage{tcolorbox}
\tcbuselibrary{breakable}
\tcbset{parbox=false, left=1ex, right=1ex}

\def\vec#1{\mathchoice{\mbox{\boldmath$\displaystyle#1$}}
  {\mbox{\boldmath$\textstyle#1$}}
  {\mbox{\boldmath$\scriptstyle#1$}}
  {\mbox{\boldmath$\scriptscriptstyle#1$}}}

\newcommand{\mech}{\ensuremath{\mathcal{A}}\xspace}
\newtheorem{theorem}{Theorem}

\newtheorem{definition}{Definition}[section]

\newcommand{\Ps}{\mathbb{P}}

\newcommand{\1}{\mathbbm{1}}

\newcommand{\D}{\mathcal{D}}

\newcommand{\X}{\mathcal{X}}

\newcommand{\A}{\mathcal{A}}
\newcommand{\Y}{\mathcal{Y}}

\def\eg{{\em e.g.,}\xspace}
\def\ie{{\em i.e.,}\xspace}
\def\cf{{\em c.f.,}\xspace}

\newcommand{\tabref}[1]{Table~\ref{#1}\xspace}

\newcommand{\resource}[1]{\textsc{#1}}
\newcommand{\DPNR}{\resource{dpnr}\xspace}

\algnewcommand\algorithmicinput{\textbf{Input:}}
\algnewcommand\Input{\item[\algorithmicinput]}
\algnewcommand\algorithmicoutput{\textbf{Output:}}
\algnewcommand\Output{\item[\algorithmicoutput]}
\algnewcommand\algorithmictier{\textbf{Role:}}
\algnewcommand\Tier{\item[\algorithmictier]}

\title{Differentially Private Representation for NLP: Formal Guarantee and An Empirical Study on Privacy and Fairness}

\author{Lingjuan Lyu$^\1$, Xuanli He$^\dagger$ \and Yitong Li$^\ast$ \\[0.5em]
$^\1$National University of Singapore, lyulj@comp.nus.edu.sg \\[0.5em]
$^\dagger$Monash University, xuanli.he1@monash.edu \\[0.5em]
$^\ast$The University of Melbourne, yitongl4@student.unimelb.edu.au \\[0.5em]
}

\begin{document}
\maketitle

\begin{sloppypar}


\begin{abstract}
It has been demonstrated that hidden representation learned by deep model can encode private information of the input, hence can be exploited to recover such information with reasonable accuracy.
To address this issue, we propose a novel approach called Differentially Private Neural Representation (DPNR) to preserve privacy of the extracted representation from text. 
DPNR utilises Differential Privacy (DP) to provide formal privacy guarantee.
Further, we show that masking words via dropout can further enhance privacy.
To maintain utility of the learned representation, we integrate DP-noisy representation into a robust training process to derive a robust target model, which also helps for model fairness over various demographic variables.
Experimental results on benchmark datasets under various parameter settings demonstrate that DPNR largely reduces privacy leakage without significantly sacrificing the main task performance.
\end{abstract}

\section{Introduction}
\label{sec:introduction}
Many language applications have involved deep learning techniques to learn text representation through neural models~\cite{bengio2003neural,mikolov2013distributed,devlin2019bert}, performing composition over the learned representation for downstream tasks~\cite{collobert2011natural,socher-etal-2013-recursive}.
However, the input text often provides sufficient clues to portray the author, such as gender, age, and other important attributes.
For example, sentiment analysis tasks often have privacy implications for authors whose text is used to train models.
Many user attributes have been shown to be easily detectable from online review data, as used extensively in sentiment analysis results~\cite{hovy2015user,potthast2017overview}.
Private information can take the form of key phrases explicitly contained in the text. However, it can also be implicit. For example, demographic information about the author of a text can be predicted with above chance accuracy from linguistic cues in the text itself~\cite{preoctiuc2015analysis}. 

On the other hand, even the learned representation, rather than the text itself, may still contain sensitive information and incur significant privacy leakage. One might argue that sensitive information like gender, age, location and password should not be leaked out and should have been removed from representation.
However, on the intermediate representation level, which is trained from the input text to contain useful features for the prediction task, it can meanwhile encode personal information which might be exploited for adversarial usages, especially a modern deep learning model has vastly more capacity than they need to perform well on their tasks.
And, it has been justified that an attacker can recover private variables with higher-than-chance accuracy, only using hidden representation~\cite{li2018towards,coavoux2018privacy}.
Therefore, the fact that representations appear to be abstract real-numbered vectors should not be misconstrued as being safe.

The naive solution of removing protected attributes is insufficient: other features may be highly correlated with, and thus predictive of, the protected attributes~\cite{pedreshi2008discrimination}.
To tackle with these privacy issues, \citet{li2018towards} proposed to train deep models with adversarial learning, which explicitly obscures individuals' private information, while improves the robustness and privacy of neural representation in part-of-speech tagging and sentiment analysis tasks.
In a parallel study, \citet{coavoux2018privacy} proposed defence methods based on modifications of the training objective of the main model. However, both works provide only empirical improvements in privacy, without any formal guarantees. Prior works have approached formal differential privacy guarantee by training differentially private deep models~\cite{abadi2016deep,mcmahan2018learning,yu2019differentially}. However, these works generally 
only considered the training data privacy rather than the test data privacy. While cryptographic methods can be used for privacy protection, it could be resource-hungry or overly complex for the user.

To alleviate the above limitations, we take inspirations from differential privacy~\cite{dwork2014algorithmic} 
to provide formal privacy guarantee of the extracted representation from user-authored text.
Meanwhile, we propose a robust training algorithm 
to derive a robust target model to maintain utility, which also offers fairness as a by-product.
To the best of our knowledge, our work is the only work to date that can provide formal differential privacy guarantee of the extracted representation, while ensuring fairness.

\textbf{Our main contributions} include:
\begin{compactitem}
\item For the first time, the privacy of the extracted neural representation from text is formally quantified in the context of differential privacy. A novel approach called Differentially Private Neural Representation (DPNR) is proposed to perturb the extracted representation.
Also, we prove that masking words via dropout can further enhance privacy.
\item To maintain utility, we propose a robust training algorithm that incorporates the noisy training representation 
in the training process to derive a robust target model, which also reduces 
model discrimination in most cases.
\item On benchmark datasets across various domains and multiple tasks, we empirically 
demonstrate that our approach yields comparable accuracy to the non-private baseline on the main task, while significantly outperforms the non-private baseline and adversarial learning on the privacy task\footnote{code and preprocessed datasets are available at: https://github.com/xlhex/dpnlp.git}.
\end{compactitem}

\section{Preliminary: Differential Privacy}

Differential privacy~\cite{dwork2014algorithmic} provides a mathematically rigorous definition of privacy and 
has become a de facto standard for privacy analysis. 
Within DP framework, there are two general settings: central DP (CDP) and local DP (LDP). 

In CDP, a trusted data curator 
answers queries or releases differentially private models by using randomisation mechanisms~\cite{dwork2014algorithmic,abadi2016deep,yu2019differentially}.
For scenarios where data are sourced from end users, and end users do not trust any third parties, DP should be enforced in a ``local" manner to enable end users to perturb their data before publication, which is termed as LDP~\cite{dwork2014algorithmic,duchi2013local}. Compared with CDP, LDP offers a stronger level of protection. 

In our system, we aim to protect the test-phase privacy of the extracted neural representations 
from end users, 
we therefore adopt LDP. 
LDP has shown the advantage that the data is randomised before 
individuals disclose their personal information, so the server and the middle eavesdropper can never see or receive the raw data. 
In terms of LDP mechanisms, randomised response~\cite{warner1965randomized,duchi2013local} and its variants have been widely used for aggregating statistics, such as frequency estimation, heavy hitter estimation, etc~\cite{erlingsson2014rappor}. 

\begin{definition}
\label{def:ldp}

Let $\mech: \mathcal{D} \to \mathcal{O}$ be a randomised algorithm mapping a data entry in $\mathcal{D}$ to $\mathcal{O}$. The algorithm $\mech$ is $(\epsilon,\delta)$-local differentially private if for all data entries $\vec{x}, \vec{x'} \in \mathcal{D}$ and all outputs $o \in \mathcal{O}$, we have
\begin{equation*}
\Pr\{\mech(\vec{x})=o\} \leq \exp(\epsilon) \Pr\{\mech(\vec{x'})=o\} + \delta
\end{equation*}
If $\delta=0$, $\mech$ is said to be $\epsilon$-local differentially private.
\end{definition}
A formal definition of LDP is provided in Definition~\ref{def:ldp}, The privacy parameter $\epsilon$ captures the privacy loss consumed by the output of the algorithm: $\epsilon=0$ ensures perfect privacy in which the output is independent of its input, while $\epsilon\rightarrow\infty$ gives no
privacy guarantee.

For every pair of adjacent inputs $\vec{x}$ and $\vec{x'}$, differential privacy requires that the distribution of $\mech(\vec{x})$ and $\mech(\vec{x'})$ are ``close'' to each other where closeness are measured by the privacy parameters $\epsilon$ and $\delta$. Typically, the inputs $\vec{x}$ and $\vec{x'}$ are adjacent inputs when 
all the attributes of one record are modified. In real scenario, the adjacent input is an application specific notion. For example, a sentence is divided into several items for every 5 words, and two sentences are considered to be adjacent if they differ by at most 5 consecutive words~\cite{wang2018not}. In this work, we consider a word-level DP, \ie two inputs are considered to be adjacent if they differ by at most 1 word.
For brevity, we use $(\epsilon,\delta)$-DP to represent $(\epsilon,\delta)$-LDP for the rest of the paper.
We remark that all the randomisation mechanisms used for CDP, including Laplace mechanism and Gaussian mechanism~\cite{dwork2014algorithmic}, can be individually used 
by each party to inject noise into local data to ensure LDP before releasing~\cite{lyu2020towards,yang2020local,lyu2020democratise,sun2020federated}. In particular, we adopt Laplace Mechanism which ensures $\epsilon$-DP with $\delta=0$ throughout the paper.

In a nutshell, data universe can be expressed as $\D = (\X, \A, \Y)$, which will be convenient to partition as $(\X, \Y) \times \A$~\cite{jagielski2018differentially}. 
Given one person's record $\vec{x}$, we can write it as a pair $\vec{x} = (\vec{x}_I, \vec{x}_S)$ where $\vec{x}_I \in (\X,\Y)$ represents the insensitive attributes and $\vec{x}_S \in \A$ represents the sensitive attributes. Our main goal is to promise differential privacy only with respect to the sensitive attributes. Write $\vec{x}_S \sim \vec{x}'_S$ to denote that $\vec{x}_S$ and $\vec{x}'_S$ differ in exactly one coordinate (i.e. one word/token in NLP domain). An algorithm is $(\epsilon,\delta)$-\emph{differentially private in the sensitive attributes} if for all $\vec{x}_I \in (\X,\Y)$ and for all $\vec{x}_S \sim \vec{x}_S' \in \A$ and for all $O \subseteq \mathcal{O}$, we have:
$$
\Ps \left[ M(\vec{x}_I, \vec{x}_S) \in O\right] \le e^{\epsilon} \, \Ps \left[ M(\vec{x}_I, \vec{x}'_S) \in O \right] + \delta
$$

\textbf{Post-processing}. DP enjoys a well-known post-processing property
~\cite{dwork2014algorithmic}: any computation applied to the output of an $(\epsilon,\delta)$-DP algorithm remains $(\epsilon,\delta)$-DP. This nice property allows the attacker to implement any sophisticated post-processing function on the privatised representation from the user, without compromising DP or making it less differentially private. 

\section{Main Framework}
\label{sec:Methodology}
\subsection{Attack Scenario}
As indicated in~\secref{sec:introduction}, uploading raw input or representations to a server takes the risk of revealing sensitive information to the eavesdropper who eavesdrops on the hidden representation and tries to recover private information of the input text.
Hence, similar as~\citet{coavoux2018privacy}, we consider an attack scenario during inference phase in Figure~\ref{fig:Attack_Scenario}, which consists of three parts: (i) a \textit{feature extractor} to extract latent representation of any test input $\vec{x}$; (ii) a \textit{main classifier} to predict the label $y$ from the extracted latent representation; (iii) and an \textit{attacker} (eavesdropper) who aims to infer some private information $\mathbf z$ contained in $\vec{x}$, from the latent representation of $\vec{x}$ used by the main classifier. In this scenario, each example consists of a triple $(\vec{x}, y, \mathbf z)$, where $\vec{x}$ is an input text, $y$ is a single label (e.g.\ topic or sentiment), and $\mathbf z$ is a vector of private information contained in $\vec{x}$.
Such attack would occur in scenarios where the computation of a neural network is shared across multiple devices. For example, phone users send their learned representations to the cloud for grammar correction or translation~\cite{li2018towards}, or to obtain the classification result, \eg the topic of the text or its sentiment~\cite{li2017privynet}.

\begin{figure}[!t]
\centering
        \includegraphics[width=8cm]{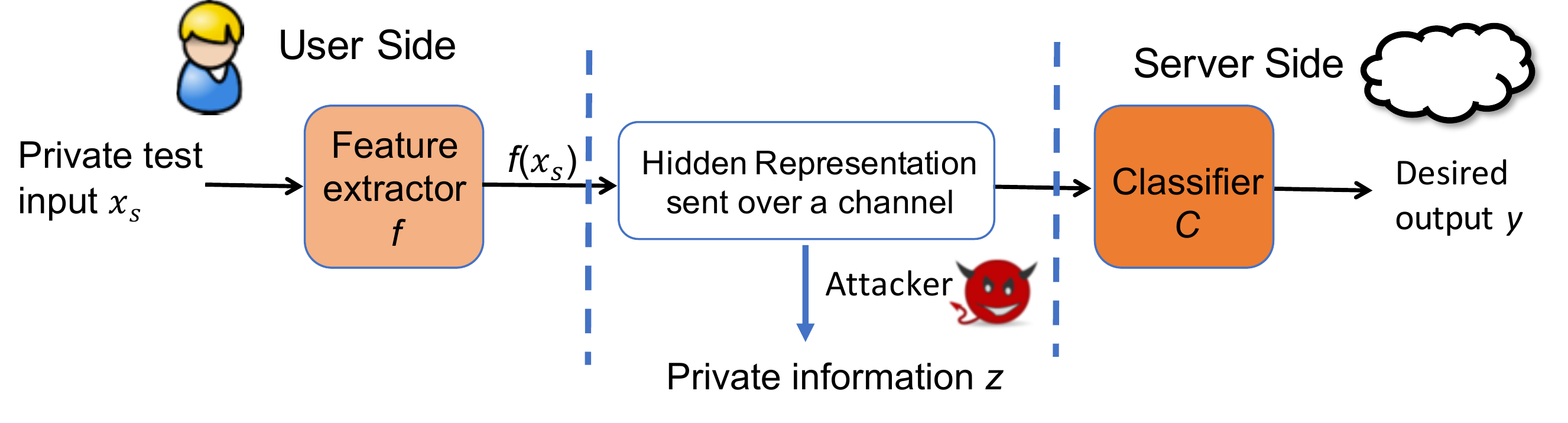}
        \caption{Attack scenario during inference phase
        .}
\label{fig:Attack_Scenario}
\end{figure}

\subsection{Methodology}
To defend against the middle eavesdropper, we aim to design an approach that can preserve privacy of the extracted test representation from the user without significantly degrading the main task performance. To achieve this goal, we introduce a DP noise layer after a predefined feature extractor (determined  by the server), which results in differentially private representation that can be transferred to the server for classification (the topic of the text or its sentiment), as shown in Figure~\ref{fig:PPNR}.

\begin{figure}[!t]
\centering
        \includegraphics[width=8cm]{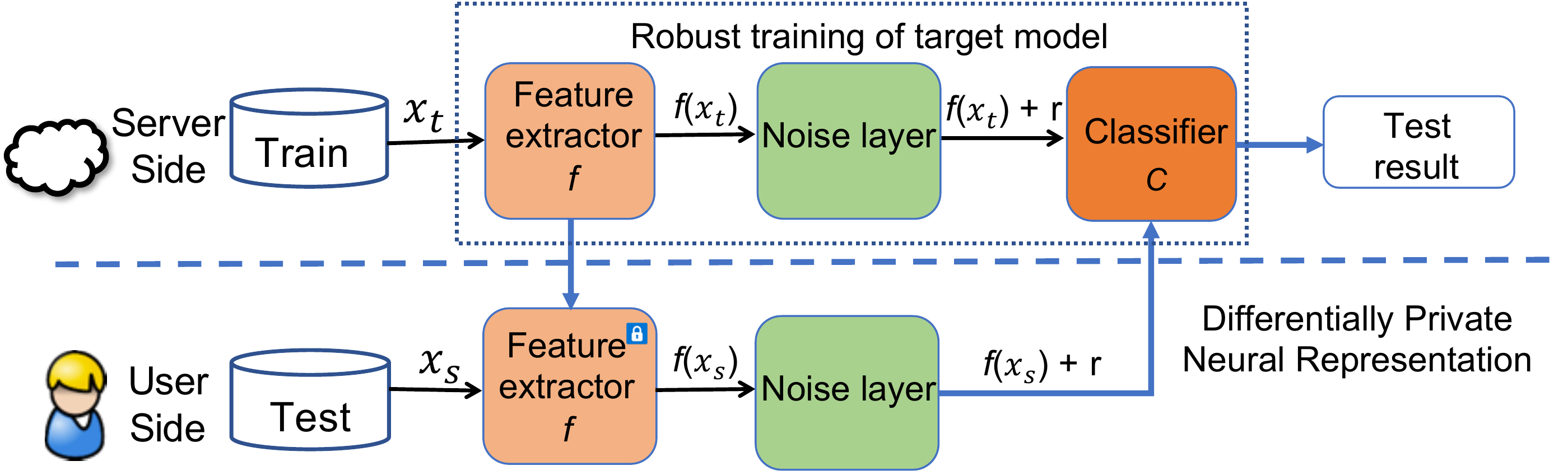}
        \caption{Illustration of the proposed main framework. 
        }
\label{fig:PPNR}
\end{figure}



In terms of model training on the server, theoretically, one could remove the noise layer and conduct non-private training by following Equation~\ref{equ:Standard_Training}:

\begin{eqnarray}
\label{equ:Standard_Training}
\mathcal{L}(\vec{x},\vec{y}) &\!\!\!=\!\!\!&\mathcal{X}(\mathrm{C}(f(\vec{x})), \vec{y})
\end{eqnarray}
where $f$ is the feature extractor, $\mathrm{C}$ is the classifier, $\vec{y}$ is the true label, and $\mathcal{X}$ denotes the cross entropy loss function.

However, doing so may deteriorate test performance, 
due to the 
injected noise in the test representation. To improve model robustness to the noisy representation, we put forward a robust training algorithm 
by incorporating a noise layer which adds the same level of noise as the test phase in the training process as well. Therefore, the robust training objective can be re-written as:
\begin{eqnarray}
\mathcal{L}(\vec{x},\vec{y}) &\!\!\!=\!\!\!&\mathcal{X}(\mathrm{C}(\vec{f(x)}+\vec{r}), \vec{y})
\end{eqnarray}

The detailed robust training process on the server is given in Algorithm~\ref{Algorithm:dp_noisy_training}. After the robust target model is built, server then provides a \emph{feature extractor} $f$ to the user, as illustrated in Figure~\ref{fig:PPNR}.

\subsection{Privacy Guarantee}
Let $f(\vec{x})=\vec{x_r} \in \mathbb{R}^k$ be the extracted representation from $\vec{x}$ by feature extractor $f$, and to apply $\epsilon$-DP to the extracted neural representation, 
we inject Laplace noise $\vec{r}$ to $\vec{x_r}=f(\vec{x})$ 
as follows: 
$$
\hat{\vec{x}}_r=\vec{x_r}+\vec{r} \, ,
$$
where the coordinates $\vec{r}=\{r_1,r_2,\cdots,r_k\}$ are i.i.d. random variables drawn from the Laplace distribution defined by $Lap(b)$, where the noise scale $b=\frac{\Delta f}{\epsilon}$, $\epsilon$ is the privacy budget and $\Delta f$ is the sensitivity of the extracted representation.

\subsubsection{Formal Privacy Guarantee}
Algorithm~\ref{Algorithm:dp_features} outlines how to derive differentially private neural representation from the feature extractor $f$. 
Each user first feeds its masked sensitive record $\tilde{\vec{x}}_s$ into a feature extractor to extract representation $\vec{x_r} \in \mathbb{R}^k$.

Note that to apply additive noise mechanism, the sensitivity $\Delta$ of the output representation $\vec{x_r}=f(\vec{x})$ needs to be determined. Estimating the true sensitivity of $\vec{x_r}$ is challenging. Instead, we follow \newcite{shokri2015privacy} to use input-independent bounds by enforcing a [0,1] range on the extracted representation, hence bounding the sensitivity of each element of the extracted representation with 1, \ie  $\Delta f=1$. Limiting the range of the extracted representation can also improve the training process by helping to avoid overfitting.


\begin{algorithm}[!tp]
\caption{Robust Training on the Server}\label{Algorithm:dp_noisy_training}
\small
\begin{algorithmic}
  \State \textbf{Input:} Training record $(\vec{x_t},\vec{y_t})$; Feature extractor $f$; Classifier $\mathrm{C}$.
  \State 1: Extraction: $\vec{x_r} \gets f(\vec{x_t})$;
  \State 2: Normalization: $\vec{x_r} \gets \vec{x_r}-min(\vec{x_r})/(max(\vec{x_r})-min(\vec{x_r}))$;
  \State 3: Perturbation: $\hat{\vec{x}}_r \gets \vec{x}_r+\vec{r}, r_i \sim Lap(b)$;
  \State 4: Calculate loss $\mathcal{L}=\mathcal{X}(\mathrm{C}(\hat{\vec{x}}_r), \vec{y_t})$ and do backpropagation to update $f$ and 
  $\mathrm{C}$.
\end{algorithmic}
\end{algorithm}


A formal statement for the privacy guarantees of Algorithm~\ref{Algorithm:dp_features} is provided in Theorem~\ref{theorem:rpdp}. 

\begin{theorem}
\label{theorem:rpdp}
Let the entries of the noise vector $\vec{r}$ be drawn from $Lap(b)$ 
with $b=\frac{\Delta f}{\epsilon}$
. Then Algorithm~\ref{Algorithm:dp_features} is $\epsilon$-differentially private.
\end{theorem}

\begin{algorithm}[t]
\caption{Differentially Private Neural Representation (DPNR)}\label{Algorithm:dp_features}
\small
\begin{algorithmic}
  \State \textbf{Input:} Each sensitive record $\vec{x_s} \in \mathbb{R}^d$; Feature extractor $f$.
  \State \textbf{Parameters:} Dropout vector $\vec{I_n} \in \{0,1\}^d$; 
  \State 1: Word Dropout: $\tilde{\vec{x}}_s \gets \vec{x_s} \odot \vec{I_n}$;
  \State 2: Extraction: $\vec{x_r} \gets f(\tilde{\vec{x}}_s)$;
  \State 3: Normalisation: $\vec{x_r} \gets \vec{x_r}-min(\vec{x_r})/(max(\vec{x_r})-min(\vec{x_r}))$;
  \State 4: Perturbation: $\hat{\vec{x}}_r \gets \vec{x}_r+\vec{r}, r_i \sim Lap(b)$;
    \State \textbf{Output:} Perturbed representation $\hat{\vec{x}}_r$. 
\end{algorithmic}
\end{algorithm}

\subsubsection{Word Dropout Enhances Privacy}
In NLP, each input is a sequence composed of words/tokens $\{w_1,\cdots,w_d\}$. 
Under word-level DP, two sentences are considered to be adjacent inputs if they differ by at most 1 word (\ie 1 edit distance). In this scenario, to lower privacy budget without significantly degrading the inference performance, we borrow the idea of nullification~\cite{wang2018not} and apply it to word dropout.

For each sensitive record $\vec{x_s}$, words are masked by a dropout operation before DP perturbation. Given a sensitive 
input $\vec{x_s}$ that consists of $d$ words, dropout performs word-wise multiplication of $\vec{x_s}$ with $\vec{I_n}$, \ie $\tilde{\vec{x}}_s \gets \vec{x_s} \odot \vec{I_n}$, where $\vec{I_n} \in \{0,1\}^d$. 
In can be either specified by users to mask the highly sensitive words or generated randomly. The number of zeros in $\vec{I_n}$ is determined by $d \cdot \mu$, where $\mu$ is the dropout rate. The zeros are located in $\vec{I_n}$ conforming to the uniform distribution.

As stated in Theorem~\ref{theorem:nullification}, word dropout in combination with any $\epsilon$-differentially private mechanism provides a tighter privacy bound in the context of word-level DP.
A detailed proof follows.

\begin{theorem}\label{theorem:nullification}
Given an input $\vec{x} \in D$, suppose $\mathcal{A}(\vec{x})=f(\vec{x})+\vec{r}$ is $\epsilon$-differentially private, let $\vec{I_n}$ with dropout rate $\mu$ be applied to $\vec{x}$, \ie $\tilde{\vec{x}} = \vec{x} \odot \vec{I_n}$, then $\mathcal{A}(\tilde{\vec{x}})$ is $\epsilon'$-differentially private, where $\epsilon'=ln[(1-\mu)\exp(\epsilon)+\mu]$.
\end{theorem}
\begin{proof}
Suppose there are two adjacent inputs $\vec{x_1}$ and $\vec{x_2}$ that differ only in the $i$-th coordinate (word), say $x_{1i} = v$, $x_{2i} \neq v$. For arbitrary binary vector $\vec{I_n}$, after dropout, $\tilde{\vec{x}}_1 = \vec{x_1} \odot \vec{I_n}$, $\tilde{\vec{x}}_2 = \vec{x_2} \odot \vec{I_n}$, there are two possible cases, \ie $I_{ni} = 0$, and $I_{ni} = 1$. 

Case 1: $I_{ni} = 0$. Since $\vec{x_1}$ and $\vec{x_2}$ differ only in $i$-th coordinate, after dropout, $\tilde{x}_{1i} = \tilde{x}_{2i} =0$, hence $\vec{x_1} \odot \vec{I_n}= \vec{x_2} \odot \vec{I_n}$. It then follows
\begin{equation*}
	\Pr\{\mathcal{A}(\vec{x_1} \odot \vec{I_n})=S\} = \Pr\{\mathcal{A}(\vec{x_2} \odot \vec{I_n})=S\}.
\end{equation*}

Case 2: $I_{ni} = 1$. Since $\vec{x_1}$ and $\vec{x_2}$ differ only in the value of their $i$-th coordinate, after dropout, $\tilde{x}_{1i} = x_{1i} = v$, $\tilde{x}_{2i} = x_{2i} \neq v$, hence $\tilde{\vec{x}}_1$ and $\tilde{\vec{x}}_2$ remain adjacent inputs that differ only in $i$-th coordinate. Because $\mathcal{A}(\vec{x})$ is $\epsilon$-differentially private, it then follows
\begin{equation*}
\resizebox{0.99\hsize}{!}{
$\Pr\{\mathcal{A}(\vec{x_1} \odot \vec{I_n})=S\} \leq \exp(\epsilon) \Pr\{\mathcal{A}(\vec{x_2} \odot \vec{I_n})=S\} \, . $
}
\end{equation*}

Combine these two cases, and use the fact that $\Pr[I_{ni} = 0] = \mu$, we have:
\begin{equation*}
\resizebox{0.99\hsize}{!}{
$
\begin{aligned}
&\Pr\{\mathcal{A}(\vec{x_1} \odot \vec{I_n})=S\} \\
&=\mu\Pr\{\mathcal{A}(\vec{x_1} \odot \vec{I_n})=S\}+(1-\mu)\Pr\{\mathcal{A}(\vec{x_1} \odot \vec{I_n})=S\} \\
&\leq \mu\Pr\{\mathcal{A}(\vec{x_2} \odot \vec{I_n})=S\}+(1-\mu)[\exp(\epsilon)\Pr\{\mathcal{A}(\vec{x_2} \odot \vec{I_n})=S\}
\\
&=[(1-\mu)\exp(\epsilon)+\mu]\Pr\{\mathcal{A}(\vec{x_2} \odot \vec{I_n})=S\}
\\
&=\exp\{ln[(1-\mu)\exp(\epsilon)+\mu]\}\Pr\{\mathcal{A}(\vec{x_2} \odot \vec{I_n})=S\}
\end{aligned}
$
}
\end{equation*}

Therefore, after dropout, the privacy 
budget is lowered to $\epsilon'=ln[(1-\mu)\exp(\epsilon)+\mu]$.
\end{proof}

Since the perturbed representation 
$\mathcal{A}(\vec{x})=f(\vec{x})+\vec{r}$ is $\epsilon$-differentially private, combining dropout beforehand, the privacy budget is lowered to $\epsilon'=ln[(1-\mu)\exp(\epsilon)+\mu]$
, hence improving privacy guarantee.
Apparently, a high value of $\mu$ has a positive impact on the privacy but a potential negative impact on the utility. In particular, when $\mu=1$, all $d$ words will be masked, which gives the highest privacy, \ie $\epsilon'=0$, but totally destroys inference performance. Hence, a smaller value of $\mu$ is preferred to trade off privacy and accuracy.

\section{Experiments}
In this section, we conduct comprehensive studies over different tasks and datasets to examine the efficacy of the proposed algorithm from three facets: 1) main task performance, 2) privacy and 3) target model fairness.
\label{sec:Performance}

\subsection{Task and Dataset}
We use two natural language processing tasks: 1) sentiment analysis and 2) topic classification, with a range of benchmark datasets across various domains.
\tabref{tab:data} summarises the statistics of the used datasets.

\subsubsection{Sentiment Analysis}
Trustpilot Sentiment dataset~\cite{hovy2015user} contains reviews associated with a sentiment score on a five
point scale, and each review is associated with 3 attributes: gender, age and location, which are self-reported by users.
The original dataset is comprised of reviews from different locations, however in this paper, we only derive \resource{tp-us} for our study.
Following \citet{coavoux2018privacy}, we extract examples containing information of both gender and age, and treat them as the private information.
We categorise ``age'' into two groups: ``under 34'' (\resource{u34}) and ``over 45'' (\resource{o45}).

\subsubsection{Topic Classification}
For topic classification, we focus on two genres of documents: news articles and blog posts. 

\paragraph{News article} We use \resource{ag} news corpus~\cite{del2005ranking}.
To ensure a fair comparison, we use the corpus preprocessed by ~\citet{coavoux2018privacy}\footnote{\url{https://github.com/mcoavoux/pnet/tree/master/datasets}.
We use both ``title'' and ``description'' fields as the input document.}.
And the task is to predict the topic label of the document, with four different topics in total.

Regarding the private information in \resource{ag}, named entities appearing in text are vulnerable to privacy leakage inferred by attackers.
In order to simulate the attack, we firstly adopt the NLTK NER system \cite{bird2009natural} to recognise all ``Person'' entities in the corpus.
Then we retain the five most frequent person entities and use them as the private information.
Due to the sparsity of name entities, each target entity only appears in very few articles.
Hence we select the examples containing at least one of these named entities to mitigate the unbalance and data scarcity.
Thus, the attacker aims to identify these five entities as five independent binary classification tasks.

\paragraph{Blog posts} We derive a blog posts dataset (\resource{blog}) from the blog authorship corpus presented~\cite{schler2006effects}.
However, the original dataset only contains a collection of blog posts associated with authors' age and gender attributes but does not provide topic annotations.
Thus we follow~\citet{coavoux2018privacy} to run the LDA algorithm~\cite{blei2003latent} with the topic number of 10 on the whole collection to identify the topic label of each document.
Afterwards, we selected posts with single dominating topic ($> 80\%$) and discarded the rest, which results in a dataset with 10 different topics.
Similar to \resource{TP-US}, the private variables are comprised of the age and gender of the author.
And the age attribute is binned into two categories, ``under 20'' (\resource{U20}) and ``over 30'' (\resource{O30}).

For all three datasets, we randomly split the preprocessed corpus into training, development and test by 8:1:1.

\begin{table}[t]
    \centering
\resizebox{\columnwidth}{!}{
    \begin{tabular}{llrrr}
    
Dataset & Private Variable & \#Train & \#Dev& \#Test \\
    \toprule
\resource{tp-us} & age, gender & 22,142  & 2,767 & 2,767\\
\resource{ag} & entity & 11,657 & 1,457 &1,457 \\
\resource{blog} & age, gender & 7,098 &887 &887\\
      
     \bottomrule
 \end{tabular}
 }
    \caption{
    Summary of three pre-processed datasets.}
    \label{tab:data}
\end{table}

\subsection{Evaluation Metrics}
\label{sec:prot}
Similar to \citet{coavoux2018privacy}, we define \textit{sentiment analysis} and \textit{topic classification} as the main tasks, whereas the inference of private information is considered as the auxiliary tasks of attackers.
Each auxiliary task is eavesdropped by one attacker.

We use accuracy to assess the performance for both main tasks.
The auxiliary tasks are evaluated via the following metrics:
\begin{compactitem}
    \item For demographic variables (\ie gender and age): $1-X$, where $X$ is the average over the accuracies of the prediction by the attacker on these variables.
    \item For named entities: $1-F$, where $F$ is the F1 score between the ground truths and the prediction by the attacker on the presence of all named entities.
\end{compactitem}

We denote the value of 1-$X$ or 1-$F$ as \textit{empirical privacy}, \ie the inverse accuracy or F1 score of the attacker, higher means better empirical privacy, \ie lower attack performance.

\begin{table}[!t]

\centering
\scalebox{0.90}{
\begin{tabular}{ccccc}
   & $\epsilon$ & \resource{tp-us} & \resource{ag} & \resource{blog} \\
\toprule
\multicolumn{1}{c}{\resource{non-priv}} & &85.53&78.75 & \textbf{97.07} \\
\midrule
\multirow{5}{*}{\DPNR} 
& 0.05 &85.65  & \textbf{80.87} &  96.69\\
& 0.1 & 85.52 & 80.78 &  96.39\\ 
& 0.5 & 85.52 & 79.71 & 96.84 \\ 
& 1 &85.36  & 79.36 & 96.39 \\ 
& 5 &\textbf{85.87}  & 79.59 & 96.66 \\ 
\bottomrule
\end{tabular}
}
\caption{Main task accuracy [\%] of \resource{non-priv} and \DPNR over 3 datasets with varying $\epsilon$ and fixed $\mu=0$.} 
\label{tbl:accuracy}
\end{table}

\subsection{Model Selection}
\label{sec:experiment_settings}
\textbf{Model and Parameters}. For implementation, owing to its success across multiple NLP tasks, we apply \textit{BERT} base ~\cite{devlin2019bert} to the classification tasks.
Specially, BERT takes a text input, then generates a representation which embeds holistic information.
We apply a dropout to this representation before a softmax layer, which is responsible for label classification.
We run 4 epochs on the training set, and choose the checkpoint with the best loss on the dev set.

After we obtain a well-trained target model, we partition it into two parts, BERT model acts as the feature extractor $f$ in Figure~\ref{fig:PPNR}, which could be deployed on users' devices, while the remaining layers act as the classifier on the server.
In our implementation, privacy is enforced in the hidden representation extracted by the feature extractor as shown by Algorithm~\ref{Algorithm:dp_features}.
For attack classifier, we utilise a 2-layer MLP with 512 hidden units and ReLU activation trained over the target model, which delivers the best attack performance on the dev set in our preliminary experiments.

We report the averaged results over 5 independent runs for all experiments.

\subsection{Performance Analysis of Target Model}
\label{sec:experiment_evaluation}
Firstly, we would like to study how the privacy parameters ($\epsilon,\mu$) in Theorem~\ref{theorem:rpdp} and \ref{theorem:nullification} affect the accuracy of main tasks. We investigate this using different parameter settings, varying one parameter while fixing the other.

\subsubsection{Impact of Privacy Budget $\epsilon$}
\label{sec:p_budget}
To analyse the impact of different privacy budget $\epsilon$ on accuracy, we choose $\epsilon \in \{0.05,0.1,0.5,1,5\}$ with fixed $\mu=0$.
Noted that to provide reasonable privacy guarantee, $\epsilon$ should be set below 10~\cite{hamm2015crowd,abadi2016deep}.
Moreover, $\epsilon \leq 1$ means a relatively tight privacy guarantee. Surprisingly, there is no obvious relationship between accuracy and $\epsilon$. We speculate the denoising training procedure of BERT and layernorm~\cite{ba2016layer} make BERT resistant to the injected noises, which can maintain the performance of the main tasks. We will conduct an in-depth study on this in the future.

Table~\ref{tbl:accuracy} shows that in most cases, our method can achieve comparable performance to the non-private baseline, across all $\epsilon$ even when the noise level is high ($\epsilon=0.05$), which validates the robustness of our method to DP noise.
It also implies that the DP-noised representation not only preserves privacy, but also retains general information for the main task.

\subsubsection{Impact of Dropout Rate $\mu$}
Similarly, we study how the word dropout rate $\mu$ affects accuracy-privacy trade-off.
\tabref{tbl:drop_accuracy} reports the performance of different models under different $\mu \in \{0.1,0.3,0.5,0.8\}$ with fixed $\epsilon=1$.
In most cases, as $\mu$ becomes larger, accuracy starts to degrade as expected. 
However, as indicated in Theorem~\ref{theorem:nullification}, higher $\mu$ results in better privacy as well. Moreover, $\mu=0.5$ can still provide a relatively high accuracy, while privacy budgets are reduced to $\epsilon'=ln[(1-\mu)\exp(\epsilon)+\mu]=0.62$.

\begin{table}[t]

\centering
\scalebox{0.90}{
\begin{tabular}{ccccc}
   & $\mu$ & \resource{tp-us} & \resource{ag} & \resource{blog} \\
\toprule
\multicolumn{1}{c}{\resource{non-priv}} & &\textbf{85.53}&78.75 & \textbf{97.07} \\
\midrule
\multirow{4}{*}{\DPNR} 
& 0.1 &85.53  &\textbf{80.71}  & 96.05 \\
& 0.3 &84.85  & 79.18 & 93.76 \\
& 0.5&83.51  &77.42  & 90.98 \\
& 0.8 & 80.70 &69.57  & 82.94 \\ 
\bottomrule
\end{tabular}
}
\caption{Main task accuracy [\%] of \resource{non-priv} and \DPNR over 3 datasets with varying $\mu$ and fixed $\epsilon=1$.} 
\label{tbl:drop_accuracy}
\end{table}

Overall, both results demonstrate that our \DPNR can protect privacy of the extracted representations of user-authored text, without significantly affecting the main task performance.

\subsection{Attack Model}
\label{sec:attack_model}
Apart from formal privacy guarantee from DP, we use the performance of the diagnostic classifier of the attackers 
for empirical privacy.
To fairly compare with the standard training and adversarial training in previous work~\cite{coavoux2018privacy}, we train an attack model that is trying to predict private variables from the representation.
We measure the empirical privacy of a hidden representation by the ability of an attacker to predict accurately specific private information from it.
If its empirical privacy (\cf Section \ref{sec:prot}) is low, then an eavesdropper can easily recover information about the input.
In contrast, a higher empirical privacy (close to that of a most-frequent label baseline) suggests that $\vec{x_r}$ mainly contains useful information for the main task, while other private information is erased.

\begin{figure}
    \centering
    \includegraphics[scale=0.33]{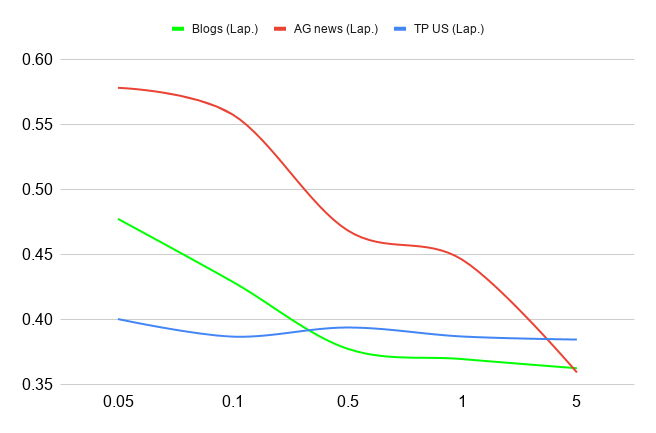}
    \caption{Results of privacy protection over \resource{tp-us}, \resource{ag} and \resource{blog} datasets across different differential privacy budgets. X-axis is the differential privacy budget $\epsilon$, while Y-axis indicates the empirical privacy (see \secref{sec:prot}).}
    \label{fig:all_privacy}
\end{figure}

To study the relationship between DP and empirical privacy, we numerically investigate the impact of the different differential privacy budgets on empirical privacy.
Recall that the empirical privacy is measured by 1-$X/F$, and the higher is better.
\figref{fig:all_privacy} shows that with the increase of the budget, empirical privacy across all datasets demonstrate a decreasing trend, especially for \resource{ag}, which well aligns with DP where the higher value of $\epsilon$ implies lower formal privacy guarantee.
Since $\epsilon=0.05$ provides the best privacy guarantee, we fix $\epsilon=0.05$ and $\mu=0$ as a default setting in the rest of this section, unless otherwise mentioned.

\paragraph{How private are the noisy neural representations?}

For empirical privacy, we investigate whether our \DPNR can provide better attack resistance compared with the adversarial learning (\resource{adv}) \cite{coavoux2018privacy} and non-private training method (\resource{non-priv}), which indicates a lower bound.
We also report the majority class prediction (\resource{majority}) as an upper bound.

\tabref{tab:priv} shows that the attack model can indeed recover private information with reasonable accuracy when targeting towards the non-private representations, manifesting that representations inadvertently capture sensitive information about users, apart from the useful information for the main task.
By contrast, our \resource{dpnr} significantly reduces the amount of information encoded in the extracted representation, as validated by the substantially higher empirical privacy than 
\resource{non-priv} across all datasets.
We also observe that our \resource{dpnr} achieves comparable empirical privacy to the majority class (\resource{majority}), and consistently outperforms the adversarial learning (\resource{adv}) from~\citet{coavoux2018privacy}, which confirms the argument of~\citet{elazar2018adversarial} that adversarial learning can not fully remove sensitive demographic traits from the data representations. 
Conversely, the post-processing property of DP ensures that the privacy loss of the extracted representation cannot be increased even by the most sophisticated attacker.

This claim can be further confirmed by \tabref{tab:ag_attack}, which reports the accuracy of the attacker on classifying whether a named entities is absent or presented in the document over \resource{ag}\footnote{For space limitation, we only report 3 of 5 entities and the results of other two are similar.}.
Generally, both \resource{adv} and \resource{dpnr} can reduce attack accuracy, misleading the attacker classifier to predict most of the shared representations as majority (A).
While our \resource{dpnr} significantly outperforms both \resource{non-priv} and \resource{adv}, corroborating our analysis above.

\begin{table}[t]
    \centering
    \resizebox{\columnwidth}{!}{%
    \begin{tabular}{lcccccc}
    

\multirow{2}{*}{} &\multicolumn{2}{c}{\resource{tp-us}} & \multicolumn{2}{c}{\resource{ag}} & \multicolumn{2}{c}{\resource{blog}}\\
\cmidrule(lr){2-3}  \cmidrule(lr){4-5} \cmidrule(lr){6-7}
       & Main &Priv.   &Main &Priv.   &Main &Priv.   \\
\toprule
\resource{majority} & 79.40  & 36.39 & 57.79  & 49.34 & 34.16  & 46.96\\
\midrule
\resource{non-priv} & 85.53  & 34.71 & 78.75  & 23.24 &  97.07 &33.88\\
\resource{adv} & -0.25 & +0.67 & -21.71 &+26.43 &-2.44  &+1.16\\
\resource{dpnr} & \bfseries +0.12 &\textbf{+3.66} & \textbf{+2.12} & \textbf{+31.13} &-0.38 &\textbf{+15.86}\\
\bottomrule

\end{tabular}%
}
\caption{Results of the main task and the privacy-protected task on the test sets over different datasets.
The relative values are based on \resource{non-priv} method and \textbf{bold} indicates our \resource{dpnr} achieves better performance than other methods. (See \secref{sec:prot} for details for metrics.)}
    \label{tab:priv}
\end{table}

\begin{table}[t]
    \centering
    \scalebox{0.74}{
    \begin{tabular}{lcccccc}
    
&  \multicolumn{2}{c}{Entity 3} & \multicolumn{2}{c}{Entity 4}& \multicolumn{2}{c}{Entity 5}\\
\cmidrule(lr){2-3} \cmidrule(lr){4-5} \cmidrule(lr){6-7}
& A & P & A & P &  A & P  \\
\toprule
ratio [\%] & 82 & 18 & 90 & 10 & 91 & 9\\
\resource{non-priv} & 96.71  & 81.99 & ~~99.43 &47.20 &~~96.93&68.29\\
\resource{adv} & 98.57  & 39.71 & 100.00 &~~0.00 &~~99.87&12.96\\
\resource{dpnr}&  90.86 & ~~8.46  &100.00  &~~0.00&100.00&~~0.00\\
      
\bottomrule
 \end{tabular}}
\caption{Accuracy of attack classifier on absence (A) and presence (P) classification of 3 entities over \resource{ag}.}
    \label{tab:ag_attack}
\end{table}



\subsection{Target Model Fairness}
Recently, fairness concern has gained lots of attention in NLP community \cite{bolukbasi2016man,zhao2017men,chang2019bias,lu2018gender,sun2019mitigating}.
Depending on the literature, fairness can have different interpretation.
In this section, we further consider the relation between differential privacy and fairness.
We ask the research question \emph{whether differential privacy noise can help enhance model fairness?}
We focus on a particular scenario of fairness, that is given a specific demographic variable (e.g. gender) a fair model should deliver an equal or similar performance over the subgroups (e.g. male vs. female) \cite{rudinger-etal-2018-gender,zhao-etal-2018-gender}.

To empirically evaluate the fairness, we take inspirations of \citet{rudinger-etal-2018-gender,zhao-etal-2018-gender,li2018towards} and partition the test data into sub-groups by the demographic variables, \ie \textbf{age, gender} and five person \textbf{entities}.
Different from predicting demographic variables in attacker (\secref{sec:attack_model}), we measure the main task accuracy difference among subgroups of demographic variables.

In fact, we noticed \DPNR can also help mitigate the bias in the representations with respect to the specific demographic or identity attributes, such that the decisions made by 
our robust target model are able to improve the fairness among the concerned demographic groups.

\begin{table}[t]
    \centering
     \scalebox{0.85}{
    \begin{tabular}{llcccc}
    
& &\multicolumn{2}{c}{Gender} & \multicolumn{2}{c}{Age} \\
\cmidrule(lr){3-4} \cmidrule(lr){5-6}
       & & F & M   & U & O   \\
\toprule
      \multirow{3}{*}{\resource{tp-us}} & ratio [\%] & 37  & 63 & 64 & 36\\
      &\resource{non-priv} &  83.69 & +1.57 & 84.63 &+0.02 \\
      &\resource{Adv.} & 84.95& +0.19&85.38 &-0.46\\
      &\resource{dpnr}&  85.90 & \textbf{+0.49}  &86.08  &+0.31\\
      \midrule
      \multirow{3}{*}{\resource{blog}} &ratio [\%] & 52  & 48 & 46 & 54\\
      &\resource{non-priv} & 98.07  & -2.18 & 97.05 & -1.49\\
      &\resource{Adv.} & 93.84& -7.54 &91.91 & -3.57\\
      &\resource{dpnr}&  98.00 &-2.34  & 97.09 &\textbf{-0.11}\\
      
     \bottomrule
 \end{tabular}}
    \caption{The accuracy of main tasks among different demographic groups ({age} and {gender}) on \resource{tp-us} and \resource{blog}. ``Ratio'' means the ratio between two subgroups of the demographic variable. The relative values ({M} and {O}) of right subgroups are deviated from the left subgroups ({F} and {U}) accordingly. }
    \label{tab:fair_1}
\end{table}

\begin{figure}[t]
  \centering
  \subfloat[\resource{non-priv}]
   {\label{fig:t_dp}\includegraphics[scale=0.18]{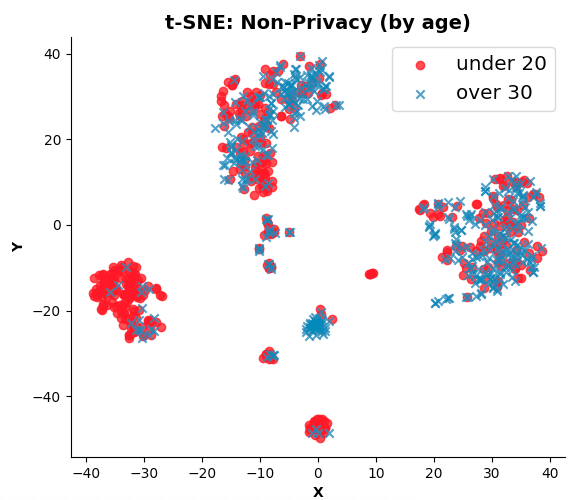}}
 ~
  \subfloat[\resource{dpnr}]
  {\label{fig:t_ndp}\includegraphics[scale=0.18]{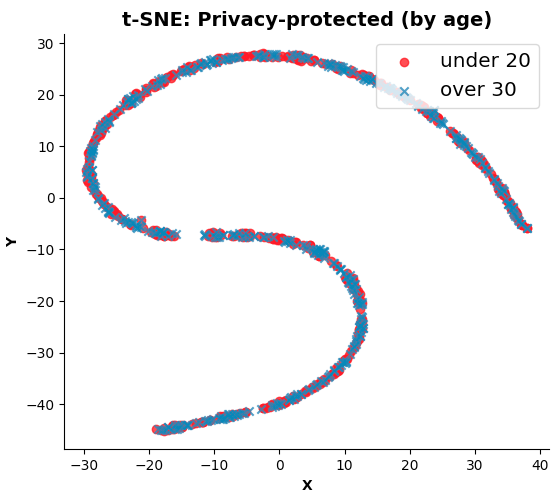}}
  \caption{t-SNE plots of the extracted representations over two {age} subgroups (\textcolor{red}{\resource{u20}} and \textcolor{blue}{\resource{o30}}) of \resource{blog} using \resource{non-priv} method and proposed \resource{dpnr}.}
   \label{fig:tsne}
\end{figure}

First of all, as the distribution of the demographic groups in \resource{tp-us} and \resource{blog} datasets is relatively even, 
hence there is no significant deviation on the main tasks (see \tabref{tab:fair_1}).
However, we still observe an noticeable difference for the \textbf{age} group in \resource{blog} and the \textbf{gender} attribute in \resource{tp-us}. 
To help better understand the phenomenon, we perform further analysis by plotting the non-private and differentially private representations of \textbf{age} on \resource{blog} in \figref{fig:tsne}.
It can be clearly observed that the patterns of two subgroups are much easier to be distinguished in the non-private representations, while the differentially private representations mostly mix the representations of ``under 20'' and ``over 30''.
We speculate that this is a consequence of the regularising effect of DP.

\begin{table}[t]
    \centering
\resizebox{\columnwidth}{!}{
    \begin{tabular}{lcccccc}
&  \multicolumn{2}{c}{Entity 3} & \multicolumn{2}{c}{Entity 4}& \multicolumn{2}{c}{Entity 5}\\
\cmidrule(lr){2-3} \cmidrule(lr){4-5} \cmidrule(lr){6-7}
& A & P & A & P &  A & P  \\
\toprule
ratio [\%] & 82 & 18 & 90 & 10 & 91 & 9\\
\resource{non-priv} & 82.67  & -18.82 & 80.86 &-15.44 &80.22&-10.78\\
\resource{Adv.} & 48.92  & +7.82 & 49.64 &+6.67 &49.43&+9.6\\
\resource{dpnr}&  83.60 &\textbf{-16.93}  &81.79  &\textbf{-12.22}&81.04&\textbf{-6.04}\\
\bottomrule
 \end{tabular}
 }
\caption{The accuracy of main tasks among three name entities on \resource{ag}. {A} means when the entity is absence, while {P} indicates the presence of that entity. The relative values are deviated from the subgroup {A}. \textbf{bold} means a statistically significant (p\textless0.0001) fairness improvement.}
    \label{tab:fair_2}
\end{table}

\tabref{tab:fair_2} shows the fairness results on \resource{ag}, where we observe the entity distributions are skewed and the prediction of the \resource{non-priv} model on the dominant groups is significantly superior to the minority groups, which causes a severe violation in terms of the fairness.
Even under such circumstance, our \DPNR 
method can mitigate this skewed bias, achieving more fair prediction than other baselines.

\section{Discussion}
Privacy and fairness are two emerging but important areas in NLP community. Prior efforts predominantly focus on either privacy or fairness~\cite{li2018towards,coavoux2018privacy,rudinger-etal-2018-gender,zhao-etal-2018-gender,lyu2020towards}, but there is no systematic study on how privacy and fairness are related. This work fills this gap, and discovers the impact of differential privacy on model fairness. We empirically show that privacy and fairness can be simultaneously achieved through differential privacy.

We hope that this work highlights the need for more research in the development of effective countermeasures to defend against privacy leakage via model representation and mitigate model bias in a general sense, and \textit{not} only specific to a particular attack. More generally, we hope that our work spurs future interest into developing a better understanding of why differential privacy works.

Meanwhile, differential privacy may incur a reduction in the model's accuracy. It is worthwhile to explore how to get a better trade-off between privacy, fairness and accuracy.

\section{Conclusion and Future Work}
\label{sec:Conclusion}
In this paper, we take the first effort to 
build differential privacy into the extracted neural representation of text during inference phase. 
In particular, we prove that 
masking the words in a sentence via dropout can further enhance privacy. To maintain utility, we propose a novel robust training algorithm that incorporates a noisy layer into the training process to produce the noisy training representation. 
Experimental results on benchmark datasets across various tasks, and parameter settings demonstrate that our approach ensures representation privacy without significantly degrading accuracy.
Meanwhile, our DP method 
helps reduce the effects of model discrimination in most cases, achieving better fairness than the non-private baseline.
Our work makes a first step towards understanding the connection between privacy and fairness in NLP -- which were previously thought of as distinct classes. Moving forward, we believe that our results justify a larger study on various NLP applications and models, which will be our immediate future work. 

\section*{Acknowledgments}
We would like to thank the anonymous reviewers for their valuable feedback.

\bibliographystyle{acl_natbib}
\bibliography{emnlp2020}

\end{sloppypar} 
\end{document}